\documentclass[10pt]{amsart}

\usepackage{amsmath, amsfonts, amssymb, amsthm, graphics, graphicx}

\newtheorem{theorem}{Theorem}[section]

\newtheorem{corollary}[theorem]{Corollary}
\newtheorem{proposition}[theorem]{Proposition}

\theoremstyle{definition}

\newtheorem{remark}[theorem]{Remark}

\numberwithin{equation}{section}

\begin{document}

\title[Conformally natural family]{Conformally Natural Families of Probability Distributions on Hyperbolic Disc with a View on Geometric Deep Learning}

\author{Vladimir Ja\'{c}imovi\'{c}}

\address{
Faculty of Sciences and Mathematics\endgraf
University of Montenegro\endgraf
D\v{z}ord\v{z}a Va\v{s}ingtona bb\endgraf
81000 Podgorica\endgraf
Montenegro\endgraf}
\email{vladimirj@ucg.ac.me}

\author{Marijan Markovi\'{c}}

\address{
Faculty of Sciences and Mathematics\endgraf
University of Montenegro\endgraf
D\v{z}ord\v{z}a Va\v{s}ingtona bb\endgraf
81000 Podgorica\endgraf
Montenegro\endgraf}
\email{marijanmmarkovic@gmail.com}



\begin{abstract}
We introduce the novel family of probability distributions on hyperbolic disc. The distinctive property of the proposed family is invariance under the actions of the group of disc-preserving
conformal mappings. The group-invariance property renders it a convenient and tractable model for encoding uncertainties in hyperbolic data. Potential applications in Geometric Deep Learning
and bioinformatics are numerous, some of them are briefly discussed. We also emphasize analogies with hyperbolic coherent states in quantum physics.
\end{abstract}

\maketitle

\section{Introduction}

Mathematically natural probability models are those that are invariant under the actions of a certain transformation group. This principle which underlies the field of integral geometry
\cite{Santalo} appears to be particularly relevant for machine learning (ML). Indeed, stochastic policies parametrized by group-invariant families have many advantages. By implementing
such policies one gains in efficiency and transparency of algorithms. 

A notable example supporting the above points is the family of Gaussian distributions. This family provides by far the most popular model in inference problems, probabilistic modeling
and stochastic search over continuous spaces. Being group-invariant, the Gaussian family has very modest representative power. Those setups in which this limitation is essential, can be
addressed by involving mixtures. The dominance of Gaussian distributions has a strong justification, as they posses a unique combination of desirable properties. One could present
information-theoretic (Gaussians are maximum entropy distributions), convex-analytic (Gaussians belong to the exponential family) and information-geometric (there is a pretty simple explicit
expression for the Fisher information) arguments in favor of this family. Still, one might argue that the most important virtue is invariance under the group of linear-affine transformations
of the vector space. The last property is (often implicitly) exploited in many algorithms where it plays a crucial role.

In conclusion, the Gaussian family justifiably holds a superior role in statistical ML over continuous spaces. The last assertion is true as long as these spaces are equipped with some flat
metric. However, many data sets coming from various fields of science, engineering and everyday life are characterized by non-zero curvatures. This means that such data are naturally embedded
into non-Euclidean spaces. It is easy to support this point by many examples. For instance, data representing rotations, or orientations in space obviously have inherent spherical geometry.
On the other side, hierarchical data (such as complex networks \cite{KPKVB}, or semantic hierarchies \cite{TBG,LW}) are naturally embedded into hyperbolic geometries. Yet another (immensely
broad) class of non-Euclidean data sets are those representing beliefs or uncertainties. The field of information geometry \cite{AJLS} sprang from the observation that the natural metric on
families of probability measures is generated by the Kullback-Leibler divergence, thus turning them into statistical (curved) manifolds.

Significance and ubiquity of non-Euclidean data has been widely recognized in data science and ML only in the last two decades. This motivated enormous research efforts, thus paving the way
for the emerging field of geometric deep learning \cite{BBCV}.      

Modeling uncertainty in non-Euclidean data requires statistical models (families of probability measures) over curved spaces. In such setups the Gaussian family becomes inadequate. The
mathematical framework for probabilistic ML over Riemannian manifolds is partially provided by directional statistics. Directional statistics is a subdiscipline within the broad field of
probability and statistics which deals with probability distributions over spheres and rotation groups \cite{MJ,PG-P}. Although it provides a solid mathematical apparatus, its significance
for ML is recognized only recently and applications are still pretty sparse. Moreover, statistical models over manifolds with negative curvature are beyond the scope of directional statistics
and are poorly investigated. This is to be changed in the near future, due to advances in ML over hyperbolic geometries \cite{GBH,NK,LLSZ}. Recent trends entail the necessity for the
corresponding statistical models in order to encode uncertainties in hyperbolic spaces.

The main goal of the present study is to introduce the family of probability distributions on hyperbolic disc which has desirable properties for probabilistic ML. The group-invariance property
facilitates implementation of algorithms and improves their transparency. For the completeness of exposition, in Section 2 we revisit the previously known family on the circle with the
group-invariance property. In Section 3 we introduce the novel family on hyperbolic disc and derive the most essential formulae. Relations with mathematical physics are exposed in Section 4.
Potential applications in ML are briefly pointed out in Section 5. Finally, Section 6 contains some concluding remarks and an outlook for the future research.

\section{The wrapped Cauchy family on the circle}

One of the most important families of probability measures on the unit circle is obtained by "wrapping" the Cauchy (Lorentzian) distributions from the real line to the circle. The "wrapping"
is realized by the Cayley transform and yields so-called {\it wrapped Cauchy distributions} which are described by the density functions \cite{McCullagh1,McCullagh2}
\begin{equation}\label{wrapped_Cauchy}
p_{wC}(\varphi) = \frac{1}{2 \pi} \frac{1-r^2}{1-2r \cos(\varphi-\Phi) + r^2}, \; \varphi \in {\mathbb S}^1.
\end{equation}
Denote this family by $wC(a)$, where $a = r e^{i \Phi}$ is a point in the unit disc. The Fisher information metric for this family coincides with the hyperbolic metric in the unit disc \cite{AG}.
Hence, the statistical manifold $wC(a)$ is isomorphic to the hyperbolic disc ${\mathbb B}^2$. 

It is easy to verify that the mean value (center of mass) of the density (\ref{wrapped_Cauchy}) is equal to $a$.

The densities (\ref{wrapped_Cauchy}) can be written in another parametrization using the complex variable $z=e^{i \varphi}$ as follows
\begin{equation*}
p_{wC}(z) = \frac{1}{2 \pi} \frac{1 - |a|^2}{|z-a|^2}.
\end{equation*}

Although $wC(a)$ is one of the central families in directional statistics, it has not been, to our best knowledge, used in ML so far. We argue that $wC(a)$ provides the most convenient statistical
model on the circle (and tori) with important favorable properties. Obviously, the family (\ref{wrapped_Cauchy}) contains the uniform density on ${\mathbb S}^1$ for $r=0$. Moreover, the Dirac
delta distributions on the circle are limit cases of (\ref{wrapped_Cauchy}) when $r \to 1$. 

In order to clarify the group-invariance property of $wC(a)$ denote by ${\mathbb G}$ the group of M\" obius transformations which leave the unit disk $\mathbb{B}^2$ invariant, i.e. the group of
conformal automorphisms of $\mathbb{B}^2$:
\begin{equation*}
\mathrm {Aut} (\mathbb {B}^2) = \left\{z\to e^{i\theta} \frac {a-z}{1- \overline{a}z}:a\in \mathbb{B}^2,\, \theta\in[0,2\pi)\right\}.
\end{equation*}

Recall that ${\mathbb G}$ is isomorphic to the matrix group $PSU(1,1) = SU(1,1)/ \pm I$.

For  $a\in \mathbb{B}^2$, denote by
\begin{equation}\label{Mobius}
g_a (z) = \frac {a-z}{1- \overline{a}z }
\end{equation}
the involutive  M\"{o}bius transformation of $\mathbb{B}^2$ (in the sense   that  $g_a\circ g_a$ is the identity transformation, and therefore $g_a^{-1} = g_a$).
Note that $g_a ( 0 ) = a$ and $g_a (a) = 0$.

\begin{proposition}
The family $wC(a)$ is invariant under the actions of the M\"{o}bius group ${\mathbb G}$. All wrapped Cauchy distributions are obtained as M\"{o}bius transformations of the uniform distribution
on the circle.
\end{proposition}

In other words, $wC(a)$ is the two-dimensional orbit of the three-dimensional Lie group ${\mathbb G}$ (where the dimension is understood in the sense of real numbers). Moreover, ${\mathbb G}$
acts transitively on this family, meaning that for any pair of measures $\mu_1$ and $\mu_2$ from $wC(a)$, there exists a transformation $g \in {\mathbb G}$, such that $\mu_2 = g_* \mu_1$. Here, the notation $g_*$ stands for the pull-back measure, defined by
$$
g_* \nu(A) = \nu(g^{-1}(A)) \mbox{  for any Borel set } A \subseteq {\mathbb S}^1.
$$

Due to the low dimension of statistical manifold $wC(a)$ the representative power of this model is very limited. Indeed, all densities (\ref{wrapped_Cauchy}) are unimodal and symmetric. Just as for the Gaussian family, the representative power can be improved by considering mixtures of the densities (\ref{wrapped_Cauchy}). 

\begin{remark}
The family of von Mises distributions \cite{MJ} is essentially the only circular statistical model that has been actively exploited in ML and bioinformatics so far. Notice that von Mises distributions are special cases of the von Mises-Fisher distributions over spheres. The von Mises family has some favorable information-theoretic properties (maximum entropy; they naturally appear as stationary distributions for diffusion processes on the circle). However, unlike wrapped Cauchy, this family does not have group-invariance properties. It is invariant only under the group of planar rotations.

We refer to \cite{Jacimovic} for a brief overview of probability distributions on circles, tori, spheres with an emphasis on their generation and (potential) applications in ML.
\end{remark}

\begin{remark}
Another family of probability distributions which is invariant under the M\" obius group is so-called Kato-Jones (K-J) family \cite{KJ}. This family gives rise to the $4$-dimensional statistical
manifold, which consists of those probability distributions that are obtained by actions of the M\" obius group on von Mises distributions. For fixed concentration parameter of the von Mises
distribution, one obtains orbits of the M\" obius group consisting of $3$-dimensional K-J sub-families. The K-J family is well investigated, with closed form expression for the density and
a number of useful formulae available \cite{KJ}. This family is more demanding to work with than the wrapped Cauchy, but has higher representative power. Wrapped Cauchy distributions constitute
the sub-manifold within the K-J family, obtained as a special case when the concentration parameter equals zero (meaning that the initial von Mises distribution is in fact uniform). The whole
$4$-dimensional K-J family contains asymmetric and bimodal densities, along with symmetric and unimodal ones. \end{remark}

\section{Conformally natural family of probability distributions over the hyperbolic disc}

Consider analytic functions on the  unit disk   generated by the family   $\{g_a:a\in\mathbb{B}^2\}$ of involutive  M\"{o}bius  transformations
\begin{equation*}\begin{split}
\mathcal {F}_{\alpha,a} &= \left\{ p_a(z) =  c_\alpha (1- |g_a(z)| ^2)^\alpha: a\in \mathbb{B}^2 \right\}
\\&= \left\{ c_\alpha \left(1- \left|\frac {a-z}{1-\overline{a}z}\right| ^2\right)^\alpha   \right\}
\\& = \left\{ c_\alpha \left(   \frac {(1- |a|^2) (1- |z|^2)}{ |1-\overline{a}z| ^2}  \right)^\alpha   \right\}
\\& = \left\{ c_\alpha(1- |a|^2)  ^\alpha  \left(   \frac {(1- |z|^2)}{ |1-\overline{a}z| ^2}  \right)^\alpha   \right\}.
\end{split}\end{equation*}

The point $a = |a| e^{i\varphi} \in {\mathbb B}^2$ and real number $\alpha>1$ are parameters of the family, while $c_\alpha$ are normalizing constants, chosen in such a way to make the above functions probability densities over ${\mathbb B}^2$:
\begin{equation}
\label{normalize}
\int _{\mathbb{B}^2} p_a(z) d\tau (z) =1,
\end{equation}
where $d \tau(z)$ denotes the element of area in hyperbolic metric. Recall that the area-hyperbolic density on the unit hyperbolic disc reads
\begin{equation*}
\tau (z) = \frac {1}{(1-|z|^2)^2}
\end{equation*}

The area-hyperbolic  measure of a  Lebesgue  set $U \subseteq \mathbb{B}^2$  is given by
\begin{equation*}
\mathrm {hyp-area}(U) = \int_U \tau(z) dA(z),
\end{equation*}
where  $dA(z)$ is the standard (Euclidean) Lebesgue   measure  on  $\mathbb{C} = \mathbb{R}^2$.

In order to calculate $c_\alpha$ we need to evaluate the integral  \eqref{normalize}. To that end we will need the following 

\begin{proposition} 
 For any automorphism $g$ of the unit disk onto itself, the following invariance formula holds:
\begin{equation*}
\int _ {\mathbb {B}^2} h (g (z)) \tau(z) dA(z)  = \int_{\mathbb {B}^2} h(w) \tau(w)dA(w),
\end{equation*}
where  $h$ is any Lebesgue measurable  mapping.
\end{proposition}

For the proof we refer to \cite[p.2]{Pavlovic}.

Using the above Proposition, we find that
\begin{equation*}\begin{split}
c_\alpha^{-1} & = \int _ {\mathbb {B}^2 } (1-|T_a (z)|^2)^\alpha \tau (z) dA(z) = \int _ {\mathbb {B}^2 } (1-|w|^2)^\alpha \tau (w) dA(w)
\\& =\int _ {\mathbb {B}^2 } (1-|w|^2)^{\alpha-2} dA(w) =  \int _ { 0 } ^1 \int _{0}^{2\pi}(1-r^2)^{\alpha-2} r dr d\theta \\&
=2\pi\int _ { 0 } ^1 (1-r^2)^{\alpha-2} r dr = \frac {\pi}{\alpha-1}.
\end{split}\end{equation*}

In conclusion, we consider the family of probability measures over ${\mathbb B}^2$ defined by the following densities
\begin{equation}
\label{conf_nat_hyp_dens}
p(z;\alpha,a) = \frac{\alpha-1}{\pi} (1- |a|^2)  ^\alpha  \left(   \frac {(1- |z|^2)}{ |1-\overline{a}z| ^2}  \right)^\alpha
\end{equation}
depending on parameters $\alpha > 1$ and $a \in {\mathbb B}^2$.

\begin{remark}
We note that
\begin{equation*}
\lim _{a\to e^{i \varphi}}c_\alpha^{-1} (1-|a|^2)^{-\alpha} p_a (z) = \mathrm {P}(|z|,\theta -\varphi),\quad  z =|z|e^{i\theta} , 
\end{equation*}
where
\begin{equation*}
\mathrm {P}(r,\Theta) =  \frac {1-r^2}{1- 2r\cos \Theta + r^2}.
\end{equation*}
\end{remark}

In the absence of better term we (following Douady and Earle \cite{DE}) refer to densities (\ref{conf_nat_hyp_dens}) as {\it conformally natural family of probability measures on the hyperbolic disc} ${\mathbb B}^2$. We will denote this family by ${\mathcal F}_{\alpha,a}$.

The family (\ref{conf_nat_hyp_dens}) is invariant under the M\" obius group. Indeed, if $\nu \in \mathcal {F}_{\alpha,a}$ and $g \in {\mathbb G}$, then obviously $g_* \nu \in \mathcal {F}_{\alpha,g(a)}$. 

The space $\mathcal {F}_{\alpha,a}$ is three-dimensional, parametrized by $\alpha>1$ and a complex number $a, \, |a|<1$. By fixing $\alpha$, we obtain two-dimensional statistical sub-manifold.

\begin{proposition}
For fixed $\alpha = \alpha_0$ the M\" obius group acts transitively on $\mathcal {F}_{\alpha_0,a}$. In other words, for any two measures $\mu_v$ and $\mu_w$ from $\mathcal{F}_{\alpha_0,a}$, there exists a M\" obius transformation $g$, such that $\mu_w = g_* \mu_v$.
\end{proposition}

\begin{proof}
Let $\mu_w$ and $\mu_v$ be two measures defined by the density (\ref{conf_nat_hyp_dens}) parametrized by points $w$ and $v$ in ${\mathbb B}^2$, respectively. Consider the M\"{o}bius transformation
$g = g_w \circ g_v^{-1} = g_w \circ g_v$, where $g_v$ and $g_w$ are defined by (\ref{Mobius}). Then, it is evident that $\mu_w = g_*\mu_v = (g_w \circ g_v)_* \mu_v$.
\end{proof} 

\begin{proposition}
Let $Z$ be the random variable with the density (\ref{conf_nat_hyp_dens}). The probability that $Z$ belongs to the disk $D_b = \{z: |z|<b\}$ for $0<b<1$ is given by the following formula
\begin{equation}
\label{prob_small}
 {\mathbb P}\{|Z|<b\} = (\alpha-1) (1- |a|^2) ^\alpha \int_{0}^{\sqrt{b}}  (1-s)^{\alpha-2} {_2}F_1 (\alpha,\alpha, 1,|a|^2 s) d s.
\end{equation} 
\end{proposition}

\begin{proof}
In the proof we will refer to the following formula (see \cite[Lemma 2.1]{Liu}):
\begin{equation}
\label{Gauss}
\frac 1{2\pi}\int _0^{2\pi} \frac {dt}{|x-e^{it}|^{2\alpha}}
= {_2}F_1 (\alpha,\alpha,1, |x|^2), \mbox{ where } x \in {\mathbb B}^2,
\end{equation}
where the notion ${_2}F_1 (\cdot,\cdot,\cdot,\cdot)$ stands for the Gauss' hypergeometric series.

Passing to polar coordinates $a  = |a| e^{i\varphi}$, and $z=r e^{i\theta}$,   we obtain the following chain of equalities: 
\begin{equation*}\begin{split}
{\mathbb P}\{z\in D_b\} & = \int_{D_b} p_a (z) d\tau(z) = \frac{\alpha-1}{\pi} \int_{D_b} \left(   \frac {(1- |a|^2) (1- |z|^2)}{ |1-\overline{a}z| ^2}  \right)^\alpha d\tau(z)
\\&= \frac{\alpha-1}{\pi}  (1- |a|^2) ^\alpha \int_{D_b}   \frac { (1- |z|^2)^{\alpha-2}}{ |1-\overline{a}z| ^{2\alpha}} dA(z)
\\&= \frac{\alpha-1}{\pi}  (1- |a|^2) ^\alpha \int_{0}^b\int_{0}^{2\pi}   \frac { (1 - r^2)^{\alpha-2}} { |1- |a| r  e^{i(\theta -\varphi)}| ^{2\alpha}} rdrd\theta
\\&= \frac{\alpha-1}{\pi}  (1- |a|^2) ^\alpha \int_{0}^b\int_{0}^{2\pi}   \frac { (1 - r^2)^{\alpha-2}} { |1- |a| r  e^{i t}| ^{2\alpha}} rdrdt
\\&= \frac{\alpha-1}{\pi}  (1- |a|^2) ^\alpha \int_{0}^b r (1-r^2)^{\alpha-2}  \int_{0}^{2\pi}  \frac {dt}   { |e^{it}- |a| r  | ^{2\alpha} }
\\&= 2 (\alpha-1) (1- |a|^2) ^\alpha \int_{0}^b r (1-r^2)^{\alpha-2} {_2}F_1 (\alpha,\alpha, 1,|a|^2 r^2) dr
\\&= (\alpha-1)  (1- |a|^2) ^\alpha \int_{0}^b  (1-r^2)^{\alpha-2} {_2}F_1 (\alpha,\alpha, 1,|a|^2 r^2) d (r^2)
\\&= (\alpha-1) (1- |a|^2) ^\alpha \int_{0}^{\sqrt{b}}  (1-s)^{\alpha-2} {_2}F_1 (\alpha,\alpha, 1,|a|^2 s) d s.
\end{split}\end{equation*}
\end{proof}

The formula    \eqref{prob_small}    is quite complicated in general case, involving the integral over Gauss' hypergeometric series. Simple expressions can be obtained in some special cases,
most notably for $\alpha=2$, which will be exposed below. 

The integral  \eqref{prob_small}   can also be easily evaluated when $a=0$. This is substantiated in the following

\begin{corollary}
Let $Z$ be the random variable with the density
$$
p(z;\alpha,0) = \frac{\alpha-1}{\pi} (1-|z|^2)^\alpha.
$$ 
Then the probability that $Z$ belongs to the disc with radius $b$ equals
\begin{equation}
\label{prob_small_alpha}
{\mathbb P}\{ |Z| < b \} = 1 - (1-\sqrt{b})^{\alpha-1}.
\end{equation}
\end{corollary}

From the formula (\ref{prob_small_alpha}) it is obvious that the distributions (\ref{conf_nat_hyp_dens}) are more concentrated for greater values of $\alpha$, with the Dirac delta distributions
arising in the limit $\alpha \to \infty$.
 
When dealing with probability measures on Riemannian manifolds, the meaningful analogue of mathematical expectation is the Riemannian center of mass (also referred to as Karcher mean). In our setup, this is a point in ${\mathbb B}^2$ defined as a unique minimizer of the following functional \cite{ABY}:
$$
J(z) = \int_{{\mathbb B}^2} d_{hyp}^2 (z,p(z;\alpha,a)) d \tau(z),
$$
where $p(z;\alpha,a)$ is defined by (\ref{conf_nat_hyp_dens}) and $d_{hyp}$ denotes the hyperbolic distance in the disc.
 
\begin{proposition} 
The Riemannian center of mass of densities (\ref{conf_nat_hyp_dens}) is $a$.
\end{proposition}

In general, computation of the Riemannian center of mass may be demanding and requires the gradient descent algorithm \cite{ABY}. However, thanks to the group-invariance of the family ${\mathcal F}_{a,\alpha}$ it suffices to apply the simple geometric argument in order to verify the above Proposition. Indeed, for $a=0$ all densities (\ref{conf_nat_hyp_dens}) are rotationally symmetric. Hence, their mean must be zero. Furthermore, the densities with the same parameter $\alpha$ are related by isometries (\ref{Mobius}) which map zero to point $a$. Then, the mean (center of mass) is mapped by by the same isometry, i.e. it equals $g(0) = a$.

\begin{figure}[!tbp]
  \centering
  \begin{minipage}[b]{0.4\textwidth}
    \includegraphics[width=\textwidth]{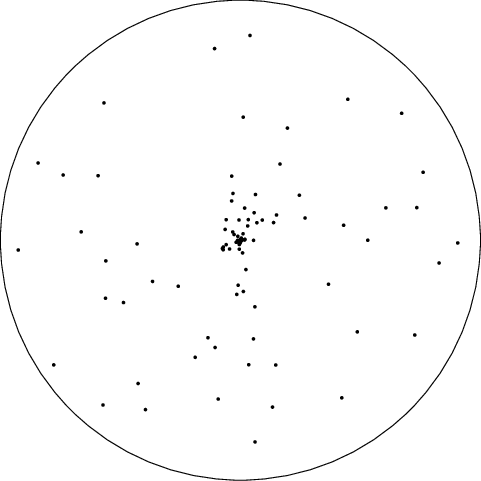}
  \end{minipage}
  \hfill
  \begin{minipage}[b]{0.4\textwidth}
    \includegraphics[width=\textwidth]{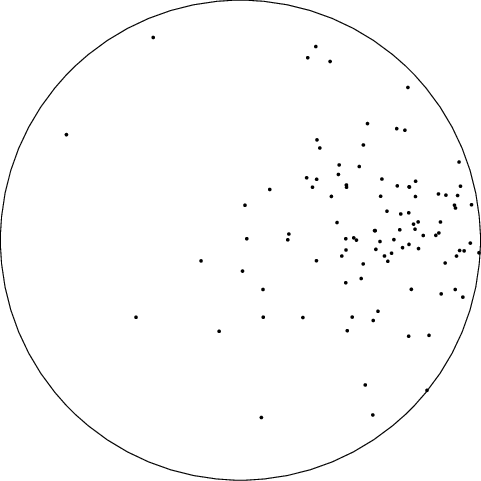}
\end{minipage}\caption{The plot of $100$ randomly generated points from $Z \sim {\mathcal F}_{a,\alpha}$ for $\alpha=2$ and $a=0$ (left), $a=1/2$ (right). }
\end{figure}

\begin{figure}[!tbp]
  \centering
  \begin{minipage}[b]{0.4\textwidth}
    \includegraphics[width=\textwidth]{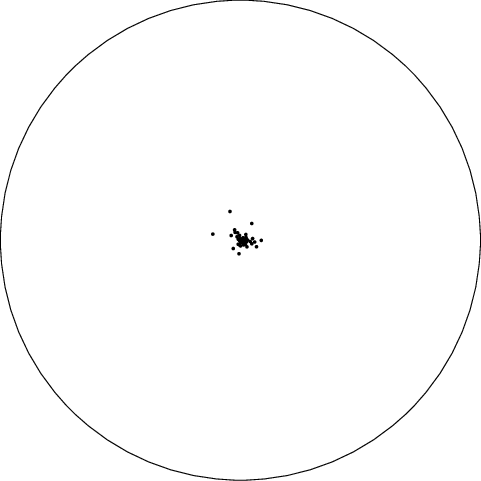}
  \end{minipage}
  \hfill
  \begin{minipage}[b]{0.4\textwidth}
    \includegraphics[width=\textwidth]{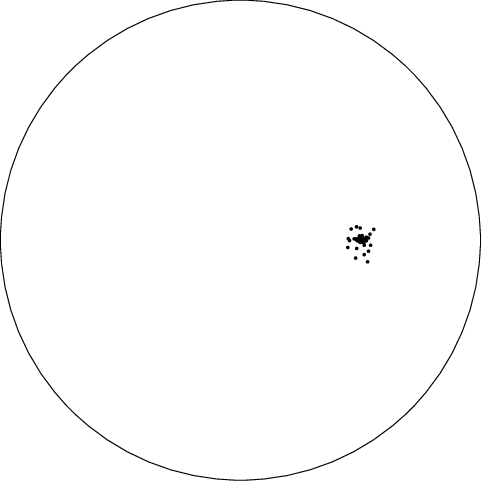}
  \end{minipage}\caption{The plot of $100$ randomly generated points from $Z \sim {\mathcal F}_{a,\alpha}$ for $\alpha=10$ and $a=0$ (left), $a=1/2$ (right).}
\end{figure}

\subsection{Random variate generation in the hyperbolic disc}

One could think of several approaches to the generation of random points from distributions ${\mathcal F}_{\alpha,a}$. The simplest method is based on Corollary 3.4. 

Recall that the M\" obius group acts transitively on sub-families ${\mathcal F}_{2,a}$ for fixed $\alpha=\alpha_0$. Therefore, it suffices to generate the random sample from the distribution with parameters $\alpha_0$ and $0$. Then, we can act on these random points by the M\" obius transformation in order to obtain the random sample from the distribution with parameter values $\alpha_0$ and $a$. 

In order to generate random points from ${\mathcal F}_{a,\alpha}$ with $a=0$ and arbitrary $\alpha$ notice that densities for $a=0$ are rotationally symmetric for all values $\alpha>1$. Therefore, the angles are uniformly distributed on $[0,2\pi]$, while the distribution of their radii is defined by the function (\ref{prob_small_alpha}).

This yields the following algorithm for generation of random variate from the distribution with parameters $\alpha$ and $a$.

1. Sample two random numbers $U_1$ and $U_2$ from the uniform distribution on $[0,1]$.

2. Set $\psi = 2 \pi U_1$.

3. Obtain the random number $\rho$ by inverting the distribution function (\ref{prob_small_alpha}), that is - by solving the equation
$$
1-(1-\sqrt{\rho})^{\alpha-1} = U_2  
$$
which yields
$$
\rho = \left( 1 - \sqrt[\alpha-1]{1-U_2}\right)^2.
$$
4. Set $z=\rho e^{i \psi}$. The random point $z \in {\mathbb B}^2$ is generated from the distribution ${\mathcal F}_{\alpha,0}$.  

5. Denote $\zeta = g_a(z)$, where $g_a$ is defined by (\ref{Mobius}). The random point $\zeta \in {\mathbb B}^2$ is generated from the distribution ${\mathcal F}_{\alpha,a}$.  

In figures 1 and 2 we depict the sets of random points from ${\mathcal F}_{\alpha,a}$ for $\alpha=2$ and $\alpha=10$, respectively, and three values of $a$. It is evident that distributions for $\alpha=10$ are more concentrated around their mean values, than those obtained for $\alpha=2$.

\subsection{The special case $\alpha = 2$}

As emphasized above, the formula (\ref{conf_nat_hyp_dens}) defines a class of families of probability measures, parametrized by $\alpha>1$. This class is invariant under the actions of the M\" obius group. However, the M\" obius group does not act transitively on the whole class, but only on sub-families obtained by fixing $\alpha$. Algorithms over these two-dimensional statistical sub-manifolds can be realized by learning M\" obius transformations (which boils down to learning generators of the M\" obius group). 

The most transparent sub-manifold is $\mathcal {F}_{2,a}$ which is parametrized by a single point $a \in {\mathbb B}^2$.

Gauss' hypergeometric series admit a simple expression for $\alpha =2$
\begin{equation*}
{_2}F_1 (2,2, 1,|a|^2 s)  = \frac{1+|a|^2s}{(1-|a|^2s)^3},
\end{equation*}
Therefore, setting $\alpha=2$ in (\ref{prob_small}) yields
\begin{equation}
\label{prob_small_2}
\begin{split}
P\{z\in D_b\} & =  (1- |a|^2) ^2 \int_{0}^{\sqrt{b}} \frac{1+|a|^2s}{(1-|a|^2s)^3} d s = \frac {(1- |a|^2) ^2  \sqrt{b}}{(1- |a|^2\sqrt{b})^2}.
\end{split}\end{equation}


In conclusion, the two-dimensional statistical manifold ${\mathcal F}_{2,a}$ provides the convenient statistical model for a wide range of problems of inference, black-box optimization and probabilistic modeling in hyperbolic disc. This particular family allows for simple and tractable formulae that can be used for sampling, parameter estimation, etc.

The family ${\mathcal F}_{2,a}$ is the orbit of the M\" obius group obtained by actions of this group on the probability measure defined by the density:
\begin{equation}
\label{ground}
p(z;2,0) = \frac{1}{\pi} (1-|z|^2)^2.
\end{equation}
Obviously, this family has very limited representative power. Again, this limitation can be addressed by considering mixtures of probability measures from ${\mathcal F}_{2,a}$. Notice that the Dirac delta distributions inside the disc can not be obtained as limiting case of densities belonging to the family $\mathcal {F}_{2,a}$. This means that those stochastic policies that are supposed to converge towards an optimal point should gradually increase the value of $\alpha$ in order to increase the concentration and thus ensure convergence towards the Dirac delta distribution.

\section{Relations with mathematical physics}

Wrapped Cauchy distributions were introduced to statistics in 1990's, \cite{McCullagh1,McCullagh2} but the expression (\ref{wrapped_Cauchy}) for their densities is familiar to mathematicians for longer than one century. If we consider (\ref{wrapped_Cauchy}) as analytic functions in the hyperbolic disc (depending on the variable $a$), they appear to be well known Poisson kernels. Poisson kernels are harmonic functions for the hyperbolic Laplace-Beltrami operator on the disc and play a central role in complex analysis and potential theory. They appear in many important formulae in mathematical physics.

In the same way, densities (\ref{conf_nat_hyp_dens}) appear as eigenfunctions of the hyperbolic Laplace-Beltrami operator in the unit disc corresponding to eigenvalues $\alpha(\alpha-1)$. This fact underpins their relevance in quantum optics \cite{DSV}.

In general quantum theory functions of the form (\ref{conf_nat_hyp_dens}) arise as $SU(1,1)$-coherent states (or hyperbolic coherent states) in analytic representations, see \cite{Perelomov,Wunsche,Vourdas}. 

Adopting the abstract framework of coherent states in quantum mechanics \cite{Perelomov}, the distribution with $a=0$ corresponds to the ground state. Indeed, densities (\ref{conf_nat_hyp_dens}) for $a=0$ are rotationally symmetric, that is - they are invariant under the maximal compact subgroup $SO(2) \sim {\mathbb S}^1$ of $SU(1,1)$. Recall that the hyperbolic disc is a homogeneous space, obtained by taking the quotient of $SU(1,1)$ over its subgroup $SO(2)$, i.e. $SU(1,1)/SO(2) = {\mathbb B}^2$. Hence, in accordance with the general group-theoretic framework of coherent quantum states \cite{Perelomov}, functions (\ref{conf_nat_hyp_dens}) are obtained by actions of the group $SU(1,1)$ on the ground state and, hence, is parametrized by points in ${\mathbb B}^2$. In our setup, the ground state corresponds to the density (\ref{ground}).

The correspondence between the family (\ref{conf_nat_hyp_dens}) and $SU(1,1)$-coherent states further implies an analogy between learning over this family and quantum evolution on invariant manifold of coherent states. Therefore, an algorithms over ${\mathcal F}_{a,\alpha}$ can be implemented by controlling the Hamiltonian of a quantum system. Since the Hamiltonian of the system is a linear combination of generators of $SU(1,1)$, its control is realized by updating the coefficients of this linear combination.

\section{Significance for ML and bioinformatics}

Recent trends in Geometric DL and data science indicate that the significance of hyperbolic data embeddings and optimization methods over hyperbolic spaces will rapidly grow both in theoretical and practical ML. This entails the necessity for tractable statistical models over hyperbolic manifolds. The family (\ref{conf_nat_hyp_dens}) provides a framework for many ML setups, such as parametrization of stochastic policies in reinforcement learning, or for the evolutionary optimization over hyperbolic (multi)discs.

There are many tasks where family (\ref{conf_nat_hyp_dens}) can be used for the design of efficient and robust algorithms. A wide range of learning algorithms dealing with uncertainties in hyperbolic data can be implemented using (mixtures or generalizations of) the family (\ref{conf_nat_hyp_dens}).

One important research direction in data science investigates and exploits embeddings of complex networks into hyperbolic manifolds. Networks are represented by so-called geometric graphs (sets of points), where the hyperbolic distance between two points is proportional to the probability that the corresponding two vertices are adjacent (i.e. that there exists an edge between them). It has been demonstrated that statistical power law (also known as Pareto-Zipf law) for the degrees of vertices implies hyperbolic geometries of complex networks and vice versa \cite{KPKVB}. On the other hand, it is well known that the vast majority of large networks arising in real life (such as internet, social networks, citations networks, biological networks) obey the general power law \cite{Newman}. This indicates that the family (\ref{conf_nat_hyp_dens}) might have many applications in statistical network science and in ML over hierarchical data of various kinds.

Since M\" obius transformations are isometries in ${\mathbb B}^2$, two sets of points which are related by a M\" obius transformation represent the same network. Therefore, one complex network corresponds to the whole sub-family ${\mathcal F}_{a,\alpha}$ for arbitrary $a \in {\mathbb B}^2$ and fixed $\alpha$. Mixtures of densities (\ref{conf_nat_hyp_dens}) can be used for probabilistic modeling of complex networks, and applied in many problems, such as comparison of the two networks, quantification of the difference/distance between them, or matching sub-networks (i.e. finding similar patterns within different networks). 

We are aware of only one study \cite{NYFK} which experimented with probability distributions on the hyperbolic disc for purposes of ML. The authors proposed the family which they named {\it wrapped normal distributions on hyperbolic disc}. They used this family for stochastic policies in various ML problems and architectures, including stochastic policies in RL (problem of playing Atari), designing hyperbolic variational auto-encoders (hyperbolic VAE's) and probabilistic word embeddings. However, we find the model proposed by the authors inconvenient and difficult to learn over. Another particular probability distribution (rather than the whole family) on hyperbolic disc is briefly mentioned in the paper \cite{KPKVB} on hyperbolic geometry of complex networks. This distribution referred to as {\it uniform on hyperbolic disc} by the authors is rotationally symmetric, while distribution of the radius is defined by the density function which is proportional to $\sinh$ (hyperbolic sine).

Experiments with densities (\ref{conf_nat_hyp_dens}) in supervised/unsupervised learning, as well as in RL with hyperbolic space of states (and, possibly, actions) are still to be conducted.

\section{Conclusion}

With recent geometric insights in data science, there is an apparent need for geometry informed ML architectures and models. Inferring symmetries and inherent curvatures of data appears as one of essential conceptual issues in this broad field. Recent developments entail a growing significance of classical mathematical fields of Riemannian geometry and Lie group theory for ML. This further implies the need for statistical models on Riemannian manifolds for encoding uncertainties in curved data.

Statistical models of data with positive curvature are provided by directional statistics and fairly well understood (although still not properly recognized and exploited by the ML research community).
  
On the other hand, data sets with inherently negative curvature seem to be even more ubiquitous in science, engineering and everyday life.  
The present study addresses the necessity for statistical models on hyperbolic manifolds. As the first step towards this goal, we propose the novel family of probability distributions on hyperbolic discs. This family is invariant with respect to actions of the group of hyperbolic isometries, which entails many favorable properties. Learning over group-invariant families boils down to learning orbits and infinitesimal generators of the corresponding group, thus facilitating implementation of stochastic policies and improving transparency of algorithms.

In practice, large hierarchical data sets are embedded into higher-dimensional manifolds such as hyperbolic balls or hyperbolic multidiscs. Therefore, extensions of the model proposed here to higher dimensions could potentially be very significant for Geometric DL.  

Notice that the wrapped Cauchy family has been extended to higher-dimensional spheres \cite{DS,KM}, thus providing a suitable statistical model (along with the von Mises-Fisher family) for spherical data. However, balls in even-dimensional (with the dimension greater than two) spaces can be equipped with two different kinds of hyperbolic metrics. Accordingly, there are two groups of hyperbolic isometries preserving the spheres \cite{Stoll}, thus giving rise to non-equivalent statistical models. We aim to continue the present study with investigations of group-invariant families on spheres and hyperbolic balls.

\end{document}